\documentclass{article}

\usepackage{ijcai17}
\usepackage{times}
\usepackage{graphicx}
\usepackage{url}

\usepackage{color,soul}
\newcommand\sr[1]{\sethlcolor{yellow}\hl{SR: #1}\sethlcolor{cyan}}
\newcommand\hg[1]{\sethlcolor{green}\hl{HG: #1}\sethlcolor{cyan}}

\newcommand\bb[1]{\sethlcolor{cyan}\hl{BB: #1}}

\usepackage{amsmath}
\usepackage{amssymb}
\usepackage{amsthm}
\usepackage{enumerate}
\usepackage{upgreek}
\usepackage[normalem]{ulem}

\usepackage[small]{caption}

\usepackage{xspace}

\newcommand{\Omit}[1]{}
\newcommand{\denselist}{\itemsep -1pt\partopsep 0pt}
\newcommand{\tup}[1]{\langle #1 \rangle}

\newcommand{\citeay}[1]{\citeauthor{#1} [\citeyear{#1}]}
\newcommand{\citeaywithnote}[2]{\citeauthor{#1} (\citeyear[#2]{#1})}
\newtheorem{definition}{Definition}

\newtheorem{theorem}[definition]{Theorem}
\newtheorem{lemma}[definition]{Lemma}

\newtheorem{corollary}[definition]{Corollary}

\def\withproofs{0} 
\newcommand{\nosuchproof}{\begin{proof}\textcolor{red}{\bf No such proof...}\end{proof}}

\newcommand{\alert}[1]{\textcolor{red}{\bf #1}}

\newcommand{\hobs}{\ensuremath{h_{obs}}\xspace}

\renewcommand{\P}{\mathcal{P}}

\DeclareMathOperator{\nextX}{\raisebox{-0.25ex}{\LARGE$\circ$}}

\DeclareMathOperator{\until}{\mathbin{\mathsf{U}}}

\DeclareMathOperator{\always}{\Box}

\DeclareMathOperator{\eventually}{\Diamond}

\newcommand{\true}{\mathsf{true}}

\title{Generalized Planning: Non-Deterministic Abstractions \\ and Trajectory Constraints}
  
\author{Blai Bonet \\
        Univ.\ Sim\'on Bol\'{\i}var \\
        Caracas, Venezuela \\
        bonet@ldc.usb.ve
        \And
        Giuseppe De Giacomo \\
        Sapienza Univ.\ Roma \\
        Rome, Italy \\
        degiacomo@dis.uniroma1.it
        \And
        Hector Geffner \\
        ICREA\\ 
        Univ.\ Pompeu Fabra \\
        Barcelona, Spain \\
        hector.geffner@upf.edu
       \And
        Sasha Rubin \\
        Univ.\ Federico II \\
        Naples, Italy \\
        rubin@unina.it
}

\begin{document}
 
\maketitle
 

\begin{abstract}
We study the characterization and computation of general policies for families of problems that 
share a  structure characterized by a common reduction into a single abstract problem. Policies $\mu$ that solve the
abstract problem  $P$ have been shown to solve all problems  $Q$ that reduce to $P$ provided that $\mu$   terminates in $Q$.
In this work, we shed light on  why this termination condition is needed  and how it can be removed.
The key observation is that the abstract problem $P$ captures the common  structure among the concrete problems $Q$
that is local (Markovian) but misses common  structure that is global. We show how such global structure can be captured
by means of  \emph{trajectory constraints} that in many cases can be expressed  as LTL formulas, thus reducing  generalized
planning    to LTL synthesis.  Moreover, for a broad class of problems that involve integer variables that can be increased or
decreased, trajectory constraints can be compiled away, reducing generalized planning to fully observable non-deterministic planning. 
\end{abstract}

\section{Introduction}

Generalized planning, where a single plan works for multiple domains,
has been drawing sustained attention in  the AI community
\cite{levesque:loops,srivastava08learning,bonet09automatic,srivastava:generalized,hu10correctness,hu:generalized,bonet:ijcai2015,BelleL16}.
This form of planning has the aim of finding generalized solutions that
solve in one shot an entire class of problems.
For example, the policy ``if block $x$ is not clear, pick up the clear
block above $x$ and put it on the table'' is general in the sense that
it achieves the goal $clear(x)$ for many problems, indeed any block
world instance. 

In this work, we study the characterization and computation of such
policies for families of problems that share a common structure. This
structure has been characterized as made of two parts: a common pool
of actions and observations \cite{hu:generalized}, and a common
structural reduction into a single abstract problem $P$ that is
non-deterministic even if the concrete problems $P_i$ are deterministic
\cite{bonet:ijcai2015}.  Policies $\mu$ that solve the abstract
problem $P$ have been shown to solve any such problem $P_i$ provided
that they terminate in $P_i$.  In this work, we shed light on why this
termination condition is needed and how it can be removed. The key
observation is that the abstract problem $P$ captures the common
structure among the concrete problems $P_i$ that is local (Markovian)
but misses common structure that is global. We show nonetheless that
such global structure can be accounted for by extending the abstract
problems with \emph{trajectory constraints}; i.e., constraints on the
interleaved sequences of actions and observations or states that are
possible.  For example, a trajectory constraint may state that a
non-negative numerical variable $X$ will eventually have value $zero$
in trajectories where it is decreased infinitely often and increased
finitely often. Similarly, a trajectory constraint can be used to  
express fairness assumptions; namely, that infinite occurrences of an
action, must result in infinite occurrences of each one of its
possible (non-deterministic)  outcomes. The language of the partially
observable non-deterministic problems extended with trajectory  
constraints provides us with a powerful framework for analyzing and
even computing general policies that solve infinite classes of
concrete problems.

\Omit{
Note that forms of abstraction are common in planning and in
verification \cite{Know94,ClGL94}. In particular the well-known
belief-sets construction itself \cite{goldman:conformant} often
adopted for solving planning under partial observability can be seen
as a form of (faithful) abstraction. In that context, \cite{DeGMRS16}
hints at the necessity of trajectory constraints, as we do here.
}


In the following, we first lay out the framework,
discuss the limitations of earlier work, and introduce projections and
trajectory constraints. 
We then show how to do generalized planning using LTL synthesis techniques, and for a
specific class of problems, using efficient planners for fully observable non-deterministic problems. 

\section{Framework} 

The  framework extends the one by \citeay{bonet:ijcai2015}.

\subsection{Model}

A \emph{partially observable non-deterministic problem} (PONDP)  is a tuple 
$P =  \tup{S,I,\Omega, Act, T,A,obs,F}$ where

\begin{enumerate}[1.]\denselist
\item $S$ is a set of \emph{states} $s$,  
\item $I \subseteq S$ is the  set of \emph{initial states},
\item $\Omega$ is the finite set of \emph{observations}, 
\item $Act$ is the finite set of \emph{actions},
\item $T \subseteq S$ is the set of \emph{goal states}, 
\item $A:S \to 2^{Act}$ is the \emph{available-actions} function, 
\item $obs:S \to \Omega$ is the \emph{observation function}, and 
\item $F: Act \times S \to 2^S \setminus \{\emptyset\}$ is the partial \emph{successor function} with domain $\{(a,s) : a \in A(s)\}$.
\end{enumerate}

In this work we assume \emph{observable goals} and \emph{observable action-preconditions}; these assumptions are made uniformly when we talk about a class of PONDPs $\cal P$:
\begin{definition}\label{dfn:class}
  A \emph{class $\cal P$ of PONDPs} consists of a set of PONDPs $P$ with 1)~the same set of actions $Act$, 2)~the same set of observations $\Omega$,
  3)~a common subset $T_\Omega \subseteq \Omega$  of \emph{goal observations} such that for all $P \in \cal P$, $s \in T$ if and only if $obs(s) \in T_\Omega$,
  and 4)~common subsets of actions $A_\omega \subseteq Act$, one  for  each observation  $\omega \in \Omega$,  such that for all  $P$ in $\cal P$ and $s$ in $P$,
  $A(s) = A_{obs(s)}$.
\end{definition}


\medskip\noindent\textbf{Example.} 
Consider the class $\P$ of problems that involve a single non-negative
integer variable $X$, initially positive, and an observation function $obs$ that determines
if $X$ is $0$ or not (written $X = 0$ and $X \neq 0$).  
The actions $Inc$ and $Dec$ increment and decrement the value of $X$ by
$1$ respectively, except when $X=0$ where $Dec$ has no effect.
The goal is  $X=0$. The problems in $\P$ differ on the initial positive value for
$X$ that may be known, e.g., $X=5$, or partially known, e.g., $X \in [5,10]$. 
\qed


PONDPs are similar to Goal POMDPs \cite{bonet:ijcai2009}  except that uncertainty is represented by sets of states rather than by probabilities.
Deterministic sensing is not restrictive \cite{chatterjee2014pomdps}. For \emph{fully observable non-deterministic problems (FONDP)}
the observation function $obs(s)=s$ is usually omitted. The  results below  apply to   fully observable problems provided that 
  $\Omega$ is  regarded as a set of \emph{state features}, and  $obs(\cdot)$ as the function that  maps  states into features.
\Omit{
Since we are interested in solutions to PONDPs that are general and apply to PONDPs with different states spaces, rather than considering
solutions as functions  that map belief states into actions, we consider instead  solutions  or policies as functions that  map observations or
streams of observations into actions. 
}

\subsection{Solutions}

Let $P=\tup{S,I,\Omega, Act, T,A,obs,F}$ be a PONDP.
A \emph{state-action sequence over $P$} is a finite or infinite sequence of the
form $\tup{s_0,a_0,s_1,a_1, \cdots}$ where each $s_i \in S$ and $a_i \in Act$.
Such a sequence \emph{reaches} a state $s$ if
$s = s_n$ for some $n \geq 0$, and is \emph{goal reaching in $P$} if it reaches a state in $T$. 
An \emph{observation-action sequence over $P$} is a finite or infinite sequence
of the form $\tup{\omega_0,a_0,\omega_1,a_1, \cdots}$ where each
$\omega_i \in \Omega$ and $a_i \in Act$. 
Let us extend $obs$ over state-action sequences: for $\tau=\tup{s_0,a_0,s_1,a_1,\ldots}$, define
$obs(\tau) = \tup{obs(s_0),a_0,obs(s_1),a_1,\ldots}$.
A \emph{trajectory of $P$} is a state-action sequence over $P$ such that 
$s_{0} \in I$, and for $i \geq 0$, $a_i \in A(s_i)$ and  $s_{i+1} \in F(a_i,s_i)$.
An \emph{observation trajectory of $P$} is an observation-action sequence of the form 
$obs(\tau)$ where $\tau$ is a trajectory of $P$.



A \emph{policy} is a partial function $\mu:\Omega^+ \rightarrow Act$ where $\Omega^+$ is the set of finite non-empty sequences over the set $\Omega$. 
A policy $\mu$ is \emph{valid} if it selects applicable actions, i.e.,
if $\tup{s_0,a_0,s_1,a_1,\ldots}$  is a  trajectory, then for the observation sequence $x_i=\tup{obs(s_0),\ldots,obs(s_i)}$,
$\mu(x_i)$ is undefined, written $\mu(x_i) = \bot$, or $\mu(x_i) \in A(s_i)$. 
A  policy that  depends on the last observation only  is called \emph{memoryless}.  We  write memoryless policies as partial functions
$\mu:\Omega \to Act$.




A  trajectory $\tau = \tup{s_0,a_0,s_1,a_1,\ldots}$ is \emph{generated} by $\mu$ if $\mu(obs(s_0) \dots obs(s_n)) = a_n$ for all $n \geq 0$ for which $s_n$ and $a_n$ are in $\tau$.
A \emph{$\mu$-trajectory of $P$} is a maximal trajectory $\tau$ generated by $\mu$, i.e., either a) $\tau$ is infinite, or b) $\tau$ is finite and cannot be extended with $\mu$; namely, 
$\tau = \tup{s_0,a_0,s_1,a_1,\ldots, s_{n-1},a_{n-1},s_n}$ and $\mu(obs(s_0) \dots obs(s_n)) = \bot$. 
If $\mu$ is a valid policy and every $\mu$-trajectory is goal reaching then we say that \emph{$\mu$ solves (or is a solution of) $P$}.
 
%
%

A \emph{transition} in $P$  is a triplet $(s,a,s') \in S \times Act \times S$ such that $s' \in F(a,s)$. 
A transition $(s,a,s')$  \emph{occurs}  in a trajectory $\tup{s_0,a_0,s_1,a_1,\ldots}$ if for some $i\geq 0$, $s=s_i$, $a=a_i$, and $s'=s_{i+1}$.
A trajectory $\tau$ is \emph{fair} if a) it is finite, or b) it is infinite and for all transitions $(s,a,s')$ and $(s,a,s'')$ in $P$ 
for which $s'\not= s''$, if one transition occurs infinitely often in $\tau$, the other occurs infinitely often in $\tau$ as well. 
If $\mu$ is a valid policy and every fair $\mu$-trajectory is goal reaching
then we say that \emph{$\mu$ is a fair solution to $P$}~\cite{cimatti:cyclic}.
\subsection{The Observation Projection Abstraction}         

  
\citeay{bonet:ijcai2015} introduce \emph{reductions} as functions that can map a set of ``concrete PONDPs problems'' $P$ into 
a single, often smaller, \emph{abstract}  PONDP  $P'$ that captures part of the common structure. 
A general way for reducing a class of problems into a smaller one is
by \emph{projecting} them onto their common observation space: 

\begin{definition}
\label{def:o-reduction}
Let $\cal P$ be a class of PONDPs (over actions $Act$, observations $\Omega$, and the set of goal observations $T_\Omega$). 
Define the FONDP  $P^o\!=\!\langle S^o,I^o,Act^o,T^o\!,A^o,F^o\rangle$,
called the \emph{observation projection of $\cal P$}, where  
\begin{enumerate}[H1.]\denselist
\item $S^o = \Omega$, 
\item $\omega \in I^o$ iff there is $P \in \cal P$ and $s \in I$ s.t.\ $obs(s) = \omega$,
\item $Act^o = Act$,
\item $T^o = T_\Omega$,
\item $A^o(\omega) = A_\omega$ for every $\omega \in \Omega$,
\item for $a \in A^o(\omega)$, define $\omega' \in F^o(a,\omega)$ iff there exists $P \in \cal P$, 
  $(s,a,s')$ in $F$ s.t.\ $obs(s)=\omega$ and  $obs(s')=\omega'$.
\end{enumerate}
\end{definition}


This definition generalises the one by \citeay{bonet:ijcai2015} which is for a single PONDP $P$, not a class $\cal P$.

\medskip\noindent\textbf{Example (continued).}
Recall the class $\cal P$ of problems with the non-negative integer variable $X$.
Their observation projection $P^o$ is the FONDP with two states, that
correspond to the two possible observations, and which we denote as $X> 0$ and $X=0$.
The transition function $F^o$ in $P^o$ for the action $Inc$ maps the 
state $X=0$ into $X>0$, and the state $X> 0$ into itself; 
and for the action $Dec$ it maps the state $X=0$ into itself, and $X>0$ into both $X>0$
and $X=0$. The initial state in $P^o$ is $X> 0$ and the goal state is $X=0$. \qed



\medskip
The following is a direct consequence of the definitions:
\begin{lemma}
\label{lemma:obsprojection}
Let $\cal P$ be a class of PONDPs and let $\tau$ be a trajectory of  some 
$P \in \cal P$. Then (1)~$\tau$ is goal reaching
in $P$ iff the observation-action sequence $obs(\tau)$ is goal reaching in $P^o$.
Furthermore, if $\mu$ is a valid policy for $P^o$, then
(2) $\mu$ is a valid policy for $P$ and (3) if $\tau$ is a
$\mu$-trajectory then $obs(\tau)$ is a $\mu$-trajectory of  $P^o$. 
\end{lemma}
\if\withproofs1
\begin{proof}
Let $\tau=\tup{s_0,a_0,s_1,a_1,\ldots}$ be a trajectory of  $P$ and
let $obs(\tau)=\tup{obs(s_0),a_0,obs(s_1),a_1,\ldots}$. 

For (1): $s_i \in T$ iff $obs(s_i) \in T_\Omega$ iff $s_i \in T^o$.

For (2): let $x_i=\tup{obs(s_1),\ldots,obs(s_i)}$ and suppose $\mu(x_i)$ is defined.
Since $\mu$ is valid for $P^o$, $\mu(x_i)\in A^o(obs(s_i))$.
Then, by H5, $\mu(x_i)\in A(s_i)$.

For (3): Suppose that $\tau$ is a $\mu$-trajectory.
By H2, $obs(s_0)$ is an initial state in $P^o$. Since $\mu(obs(s_0)) = a_0$ 
and $\mu$ is valid for $P^o$, we have that $a_0 \in A^o(obs(s_0))$. 
Hence, by H6, $obs(s_1)\in F^o(a_0,obs(s_0))$.
Repeat the argument to obtain that $obs(\tau)$ is a trajectory of  $P^o$ generated by $\mu$.
To see that $obs(\tau)$ is a maximal trajectory of  $P^o$ generated by $\mu$. Suppose $obs(\tau)$ were finite, say
$obs(\tau) = \tup{obs(s_0), a_0, \cdots, obs(s_{n-1}),a_{n-1},obs(s_n)}$, and $\mu(obs(s_0) \dots obs(s_n)) \neq \bot$. Then 
$\tau$ would not be a maximal trajectory of  $P$ generated by $\mu$.
\end{proof}
\fi

\subsection{General Policies}

\Omit{
  To motivate our work, we compare our setup with that of \citeay{bonet:ijcai2015} who define 
$P^o$ for a single PONDP (and not for a class $\cal P$, as we do).
Moreover, using the terminology in that paper, Definition~\ref{def:o-reduction} above says that the every $P \in \cal P$ 
reduces to $P^o$ through the function $h:s \mapsto obs(s)$.}
A main result of \citeaywithnote{bonet:ijcai2015}{Theorem~5} is that:

\begin{theorem}
\label{thm:bg2015}
Let $\mu$ be a fair solution for the projection $P^o$ of the single problem $P$.
If all the $\mu$-trajectories in $P$ that are fair are also \emph{finite}, then $\mu$ is a fair solution for $P$.
\end{theorem}

One of the main goals of this work is to remove the termination (finiteness) condition so that
a policy $\mu$ that solves the projection $P^o$  of a class of problems $\cal P$
will generalize automatically to all problems in the class.  For this, however, we will need to extend the
representation of abstract problems. 

\medskip\noindent\textbf{Example (continued).}
The policy $\mu$ that maps the observation $X > 0$ into the
action $Dec$ (namely, ``decrease $X$ if positive'') is a fair solution to $P^o$, and to 
each problem $P$ in the class $\P$ above. The results of \citeay{bonet:ijcai2015}, however,  are
not strong enough for establishing this generalization. What they show instead is that the generalization holds
provided that $\mu$ terminates at each problem $P\in\P$.\qed

\section{Trajectory Constraints}


In the example, the finiteness condition is required by Theorem~\ref{thm:bg2015} because
there are other problems, not in the class $\P$, that also reduce to $P^o$ but which are not solved by $\mu$.
Indeed, a problem $P'$ like any of the problems in $\P$ but where the $Dec$ action
increases $X$ rather than decreasing it when $X=5$ is one such problem: one that cannot be solved
by $\mu$ and one on which $\mu$ does not terminate  (when, initially, $X\geq 5$).
These other problems exist because the abstraction represented by the non-deterministic problem $P^o$ is too weak:
it captures some properties of the problems in $\cal P$ but misses  properties that should
not be abstracted away. One such property  is that
\emph{if a  positive  integer variable  keeps being decreased, while being increased a finite number of times only, then the value of the  variable will infinitely often reach the zero value.}\footnote{The stronger consequent of eventually reaching the zero value and staying at it also holds for the problems in $\P$. However, we prefer to use this weaker condition as it is enough for our needs.}
This property is true for the  class of problems $\cal P$ but  is not true for the problem $P'$, 
and  crucially, it is not true in the projection $P^o$. Indeed, the projection $P^o$ captures
the local (Markovian) properties of the problems in $\cal P$ but  not the  \emph{global} properties.
Global properties involve not just transitions but   \emph{trajectories.}


We will thus enrich the language of the abstract problems by considering \emph{trajectory constraints} $C$, i.e.,
restrictions on the set of possible observation or state \emph{trajectories}. 
The constraint $C_X$ that a non-negative variable $X$ decreased infinitely often and increased finitely often,
will infitenly often reach the value $X=0$ is an example of a trajectory constraint.
Another example is the fairness constraint $C_F$; namely, that  infinite occurrences of a non-deterministic  action
in a state imply infinite occurrences of each one of its possible outcomes following the action
in the same state. 

%


\subsection{Strong Generalization over Projections}


Formally, a \emph{trajectory constraint $C$ over a PONDP $P$} is a set of infinite state-action sequences over $P$ or a set of infinite observation-action sequences over $P$.
A trajectory $\tau$ \emph{satisfies} $C$ if either a) $\tau$ is finite, or b)
if $C$ is a set of state-action sequences then $\tau \in C$, and if $C$ is a set of observation-action sequences then $obs(\tau) \in C$. 
Thus trajectory constraints only constrain infinite sequences.

A PONDP $P$ extended with a set of trajectory constraints $C$, written $P/C$, is called a  \emph{PONDP with constraints}. 
Solutions of $P/C$ are defined as follows:

\begin{definition}[Solution of PONDP with constraints]
A policy $\mu$ solves $P/C$ iff $\mu$ is valid for $P$ and every $\mu$-trajectory of $P$ that satisfies $C$ is goal reaching.
\label{def:P/C}
\end{definition}

This notion of solution does not assume fairness as before (as in strong cyclic solutions);  instead,
it requires that the goal be reached along all the trajectories that satisfy the constraints.
Still if $C=C_F$ is the fairness constraint above, a \emph{fair} solution to $P$ is nothing else than
a policy $\mu$ that solves $P$ extended with $C_F$; i.e.\ $P/C_F$. When $C$ contains all trajectories, 
the solutions to $P$ and $P/C$ coincide as expected. 

We remark that a trajectory constraint $C$ over the observation projection 
$P^o$ of $\cal P$ is a trajectory constraint over every $P$ in $\cal P$. 
The next theorem says that a policy that solves $P^o/C$ also solves every $P/C$.



\begin{theorem}[Generalization with Constraints]
\label{thm:gen+constraints}
If $P^o$ is the observation {projection} of a class of problems $\mathcal{P}$
and $C$ is a trajectory constraint over $P^o$, then a policy that solves $P^o/C$ also 
solves $P/C$ for all $P \in \mathcal{P}$. 
\end{theorem}
\if\withproofs0
\begin{proof}
Suppose that $\mu$ solves $P^o/C$. Let $P$ be a problem in $\mathcal{P}$. 
By Lemma~\ref{lemma:obsprojection} (2), $\mu$ is a valid policy for $P\in\mathcal{P}$.
To see that $\mu$ solves $P/C$, take a trajectory $\tau$ of $P$ that satisfies $C$,
i.e., $obs(\tau) \in C$. 
By Lemma~\ref{lemma:obsprojection} (3), $obs(\tau)$ is a $\mu$-trajectory of $P^o$.
By assumption $obs(\tau)$ is goal reaching.
By Lemma~\ref{lemma:obsprojection} (1), $\tau$ is goal reaching.
\end{proof}
\fi

Say that \emph{$P$ satisfies the constraint $C$} if all infinite trajectories in $P$ satisfy $C$.
Then, an easy corollary is that:


\begin{corollary}
\label{cor:gen+constraints:1}
Let $P^o$ be the observation projection of a class $\cal P$ of PONDPs, let $C$ be a trajectory constraint over $P^o$,
and let $P$ be a problem in $\mathcal{P}$.
If $P$ satisfies $C$, then a policy that solves $P^o/C$ also solves $P$.
\end{corollary}
\if\withproofs1
\begin{proof}
Use Theorem~\ref{thm:gen+constraints} and the simple fact (that follows from Definition~\ref{def:P/C})
that if $\mu$ solves $P/C$ and every trajectory of $P$ satisfies $C$ then $\mu$ solves $P$.
\end{proof}
\fi

More generally, if $C$ and $C'$ are constraints over $P$, say that $C$ \emph{implies} $C'$ if every infinite trajectory over $P$ 
that satisfies $C$ also satisfies $C'$. We then obtain that:

\begin{corollary}
\label{cor:gen+constraints:2}
Let $P^o$ be the observation projection of a class $\cal P$ of PONDPs, and let $C'$ be a trajectory constraint over $P^o$.
If $P$ is a problem in $\cal P$, $C$ is a constraint over $P$, and $C$ implies $C'$,
then a policy that solves $P^o/C'$ also solves $P/C$.
\end{corollary}
\if\withproofs1
\begin{proof}
Use Theorem~\ref{thm:gen+constraints} and the simple fact (that follows from Definition~\ref{def:P/C})
that if $\mu$ solves $P/C'$ and every trajectory satisfying $C$ satisfies $C'$, then $\mu$ solves $P/C$.
\end{proof}
\fi

Trajectory constraints in $P^o$ are indeed powerful and can be used to account for all the solutions of a class $\cal P$.

\begin{theorem}[Completeness] 
\label{thm:completeness}
If $P^o$ is the observation projection of a class of problems $\cal P$, and $\mu$ is a policy that solves all the problems in $\cal P$,
then there is a constraint $C$ {over $P^o$} such that $\mu$ solves $P^o/C$. 
\end{theorem}

\begin{proof}[Sketch]
Define $C$ to be the set of observation sequences of the form $obs(\tau)$ where $\tau$ varies over all the trajectories of all $P \in \cal P$. 
Suppose $\mu$ solves every $P$ in $\cal P$.  First, one can show that $\mu$ is valid in $P^o$ using the assumption that action preconditions are observable (i.e., $A^o(\omega) = A_{\omega} = A(s)$ for all $s$ such that $obs(s) = \omega$). Second, to see that $\mu$ solves $P^o/C$ take a $\mu$-trajectory $\tau$ of $P^o$ 
that satisfies $C$. By definition of $C$, there is a trajectory $\tau'$ in some  $P \in \cal P$ such that $obs(\tau') = \tau$. It is not hard to see that $\tau'$ is a $\mu$-trajectory of $P$. Thus, since $\tau'$ is goal-reaching and goals are uniformly observable (Definition~\ref{dfn:class}),  then $\tau$ is goal reaching.
\end{proof}

\if\withproofs1
\begin{proof}

Define $C$ to be the set of observation sequences of the form $obs(\tau)$ where $\tau$ varies over all the trajectories of all $P \in \cal P$. 
Suppose $\mu$ solves every $P$ in $\cal P$. 

First, $\mu$ is valid in $P^o$. Indeed, let $\tau = \tup{w_0,a_0, \cdots}$ be a trajectory of $P^o$, 
and let $x_i = \tup{w_0, w_1, \cdots, w_i}$. Suppose $\mu(x_i)$ is defined. We must show that $\mu(x_i) \in A^o(w_i)$. By the assumption that action-preconditions are uniformly observable (Definition~\ref{dfn:class}), we have that $A(s) = A_{obs(s)}$ for every $s$ in every $P \in \P$. Thus, by H5, $A^o(w_i) = A_{w_i}$. 
Take any trajectory $\tau' = \tup{s_0, a_0, \cdots}$ in any $P \in \cal P$ such that
$obs(\tau') = \tau$ (such a $\tau'$ exists by definition of $C$). Since $\mu$ is valid in $P$, and $\mu(x_i)$ is defined, $\mu(x_i) \in A(s_i)$. 
But $A(s_i) = A_{obs(s_i)} = A_{w_i}$.

Second, $\mu$ solves $P^o/C$. To see this, take a $\mu$-trajectory $\tau = \tup{w_0,a_0, \dots}$ of $P^o$ that satisfies $C$. By definition of $C$, there is a trajectory $\tau'$ in some  $P \in \cal P$ such that $obs(\tau') = \tau$. It is sufficient to show that $\tau'$ is a $\mu$-trajectory (for then it is goal reaching, and thus so is $\tau$).

To see that $\tau'$ is generated by $\mu$ simply note that $\mu(obs(s_0) \dots obs(s_n)) = a_n$ (since $\mu$ generates $\tau$). To see that $\tau'$ is a maximal trajectory in $P$ generated by $\mu$, suppose it were finite, say $\tau' = \tup{s_0,a_0, \cdots, s_{n-1},a_{n-1},s_n}$. Then $\tau = \tup{obs(s_0),a_0, \cdots,  obs(s_{n-1}),a_{n-1}, obs(s_n)}$ is also finite. But since $\tau$ is a maximal trajectory of $P^o$ generated by $\mu$, we have that $\mu(obs(s_0) \cdots obs(s_n)) = \bot$, as required.
\end{proof}
\fi

\medskip\noindent\textbf{Example (continued).}
The problems $P$  above with  integer variable $X$
and actions $Inc$ and $Dec$ and goal $X=0$, are all solved by the policy $\mu$ 
``if $X > 0$, do $Dec$''. However, the observation projection $P^o$ is not solved
by $\mu$ as there are trajectories where the outcome of the action $Dec$ in $P^o$, with non-deterministic
effects $X > 0 \,|\, X=0$, is always $X > 0$. \citeay{bonet:ijcai2015} deal with this by taking  $\mu$
as a  \emph{fair} solution to $P$ and then proving termination of $\mu$ in  $P$
(a similar approach is used by \citeay{srivastava:aaai2011}). The theorem and corollary  above provide an alternative.
The policy $\mu$ does not solve $P^o$ but solves $P^o$ with constraint $C_X$, 
where $C_X$ is the trajectory constraint that
states that \emph{if the action $Dec$ is done infinitely often and the action $Inc$ is done finitely often then infinitely often $X=0$}.
Theorem~\ref{thm:gen+constraints}  implies that $\mu$  solves $P/C_X$ for every $P$ in the class. 
Yet, since $P$ satisfies $C_X$, Corollary~\ref{cor:gen+constraints:1} implies that $\mu$ must solve $P$ too.
The generalization also applies to problems $P$ where increases and decreases in the variable $X$ are  \emph{non-deterministic}
as long as no decrease can  make $X$ negative. In such a case, if $C_F$ is the trajectory constraint over $P$ that says that non-deterministic
actions are \emph{fair}, from the fact that $C_F$ implies $C_X$ in any such problem  $P$, Corollary~\ref{cor:gen+constraints:2} implies that a policy $\mu$ that
solves $P^o/C_X$ must also solve $P/C_F$; i.e., the strong solution to the abstraction $P^o$ over the constraint $C_X$,
represents a \emph{fair} solution to such non-deterministic problems $P$. \qed

\section{Generalized Planning as LTL Synthesis} \label{sec:LTL}



Suppose that we are in the condition of Theorem~\ref{thm:gen+constraints}. That is we have a class of problems $\cal P$ whose observation projection is $P^o$, and  a constraint $C$ over $P^o$. 
Let's further assume that $C$ is expressible in Linear-time Temporal Logic (LTL).  
Then we can provide an actual policy that solves $P^o/C$, and hence solves $P/C$ for all $P \in {\cal P}$. We show how in this section.



For convenience, we define the syntax and semantics of LTL over a set $\Sigma$ of alphabet symbols (rather than a set of atomic propositions).
The syntax of LTL is defined by the following grammar: 
$\Psi ::= 
\true \mid 
l \mid 
\Psi \wedge \Psi \mid
\neg \Psi \mid
\nextX \Psi \mid
\Psi \until \Psi$, where $l \in \Sigma$.
We denote infinite strings $\alpha \in \Sigma^\omega$ by $\alpha = \alpha_0 \alpha_1 \cdots$, and write $\alpha_{\geq j} = \alpha_j \alpha_{j+1} \dots$. 
The semantics of LTL, $\alpha \models \Psi$, is defined inductively as follows: $\alpha \models \true$; $\alpha \models l$ iff $\alpha_0 = l$; $\alpha \models \Psi_1 \wedge \Psi_2$ iff $\alpha \models \Psi_i$, for $i = 1,2$; $\alpha \models \neg \Psi$ iff $\alpha \not \models \Psi$; $\alpha \models \nextX \Psi$ iff $\alpha_{\geq 1} \models \Psi$; and $\alpha \models \Psi_1 \until \Psi_2$ iff there exists $j$ such that $\alpha_{\geq j} \models \Psi_2$ and for all $i < j$ we have that $\alpha_{\geq i} \models \Psi_1$. We use the usual shorthands, e.g., $\eventually \Phi$ for $\true \until \Phi$ (read ``eventually'') and $\always \Phi$ for $\neg \eventually \neg \Phi$ (read ``always''). Define 
$mod(\Psi) = \{\alpha \in \Sigma^\omega : \alpha \models \Psi\}$. 

Assume that the trajectory constraint $C$ is \emph{expressed} as an LTL formula $\Psi$, i.e., $\Sigma = \Omega \cup Act$ and $mod(\Psi) = C$.
%
Let  $\Phi \doteq \Psi \supset \eventually T^o$ where $\eventually T^o$ is the reachability goal of $P^o$ expressed in LTL.
To build a policy solving $P^o/C$ proceed as follows. The idea is to think of policies
$\mu:\Omega^+ \to Act$ as ($\Omega$-branching $Act$-labeled)
trees, and to build a tree-automaton accepting those 
policies such that every $\mu$-trajectory satisfies the formula
$\Phi$. Here are the steps: 

\begin{enumerate}[1.]\denselist
\item Build a nondeterministic B\"uchi automaton $A_b$ for the formula $\Phi$ (exponential in $\Phi$)~\cite{VW94}.
\item Determinize $A_b$ to obtain a deterministic parity word automaton (DPW) $A_d$ that accepts 
the models of $\Phi$ (exponential in $A_b$, and produces linearly many priorities)~\cite{Piterman07}. 
An infinite word $\alpha$ is \emph{accepted} by a DPW $A$ iff the largest priority of the states visited infinitely often by $\alpha$ is even.
\item Build a deterministic parity tree automaton $A_t$ that accepts a policy $\mu$ iff every $\mu$-trajectory satisfies $\Phi$ (polynomial in $A_d$ and $P^o$, and no change in the number of priorities). This 
 can be done by simulating $P^o$ and $A_d$ as follows: from state $(s,q)$ (of the product of $P$ and $A_d$) and reading action $a$, launch for each $s' \in F(a,s)$ a copy of the automaton in state $(s',q')$ in direction $s'$ where $q'$ is the updated state of $A_d$.
\item Test $A_t$ for non-emptiness (polynomial in $A_t$ and exponential in the number of priorities of $A_t$)~\cite{Zielonka98}.
\end{enumerate}
This yields the following complexity (the lower-bound is inherited from \citeay{PR1989b}):

\begin{theorem}
Let $P^o/C$ be the observation projection  with trajectory constraint $C$ expressed as the LTL formula $\Psi$. Then solving $P^o/C$ (and hence all  $P/C$ with $P \in {\cal P}$) is \text{2EXPTIME}-complete. In particular, it is double-exponential in $|\Psi|+|T^o|$ and polynomial in $|P^o|$.
\end{theorem}

This is a worst-case complexity. In practice, the automaton $A_t$ may be small also in the formula $\Phi \doteq \Psi\supset \eventually T^o$.

\begin{figure}[t]
\centering
\includegraphics[scale=.575]{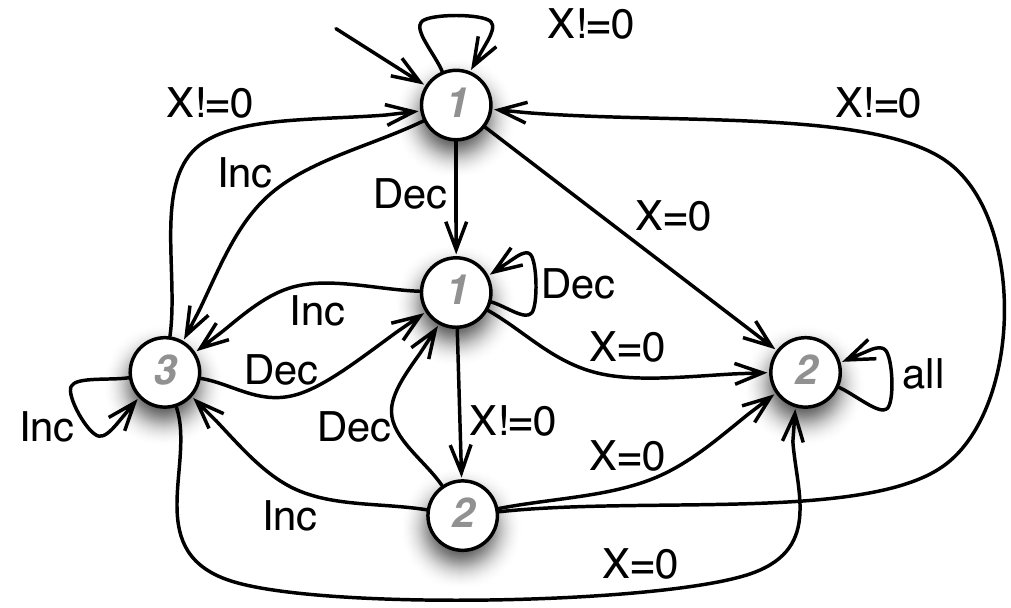}
\caption{DPW (with priorities written in the states) for
$\left[\eventually \always \neg Inc \ \wedge\ \always \eventually Dec \ \supset\   \always \eventually (X = 0) \right]\supset \eventually (X=0)$.
}\label{fig:dpw}
\end{figure}


\medskip\noindent
\textbf{Example (continued).} 
We express $C_X$ in LTL. 
The constraint LTL formula $\Psi_X$ (over alphabet $\{X=0,X \neq 0\} \cup \{Inc,Dec\})$  is
$\label{eq:ltl}
\eventually \always \neg Inc \ \wedge\ \always \eventually Dec \ \supset\   \always \eventually (X = 0)\,.
$
Then, $mod(\Psi_X) = C_X$. 
The resulting formula $\Phi \doteq \Psi_X \supset \eventually (X = 0)$ generates a relatively small automaton: the DPW $A_d$ has 5 states and 3 priorities, see Figure~\ref{fig:dpw}. 
As a further example, consider the case of $N \in
\mathbb{N}$ variables. Each variable $X_i$ has a formula $\Psi_i$
analogous to $\Psi_X$. Consider the constraint $\Psi = \wedge_i
\Psi_i$.  
The above algorithm gives us, as worst-case, DPW for $\Phi \doteq \Psi \supset \eventually (X = 0)$
of size $2^{2^{O(N)}}$ with $2^{O(N)}$ many priorities. However, 
analogously to Figure~\ref{fig:dpw}, there is a DPW of
size $2^{O(N)}$ with $3$ priorities. \qed

%

\medskip

\section{Qualitative Numerical Problems}

%
%

To concretize, we  consider a simple but broad class of  problems where
\emph{generalized planning} can be reduced to \emph{FOND planning.} 
These problems  have a set  of non-negative \emph{numeric variables},
that do not have to be integer, and standard boolean propositions and actions that can increase or decrease the value of
the numeric variables non-deterministically. The general problem of stacking a block $x$ on a block $y$ in any blocks-world problem,
with any number of blocks in any configuration,
can be cast as a problem of this type.
\Omit{
  with  two variables $N(x)$ and $N(y)$,
representing the number of blocks above $x$ and $y$, and booleans $H(z)$ and $E$ representing the block $z$ being held if any,
and whether the arm is empty or not.  Actions that place blocks above $x$ increase $N(X)$ and those that remove blocks
above $x$ decrease $X$. When both variables are zero, the block $x$ can be picked up and placed on $y$.}
An abstraction of some of these problems appears in \cite{srivastava:aaai2011,srivastava:aaai2015}.

A \emph{qualitative numerical problem} or QNP  is a tuple $R_V=\langle F,Init,Act,G,V,Init_V,INC,$ $DEC \rangle$
where the first four components define a STRIPS planning problem extended with
atoms $X=0$ and $X>0$ for numerical variables $X$ in a set  $V$,  that may appear as action preconditions  
and goals, sets $Init_V(X)$ of possible initial values for each variable $X\in V$,
and  effect descriptors  $Dec(X)$ and $Inc(X)$, with a semantics given by the functions
$INC$ and $DEC$  that take a variable $X$, an action $a$, and a state $s$,
and yield a \emph{subset} $INC(X,a,s)$ and $DEC(X,a,s)$ of possible values $x'$ of $X$ in the
successor states. If these ranges contain a single number, the problem is deterministic, else
it is \emph{non-deterministic}. If $x$ is the value of $X$ in $s$, the only \emph{requirement} is that for all
$x'\in DEC(x,a,s)$, $0\leq x'\leq x$, and that for all $x'\in INC(X,a,s)$,
$x'>x$ when $x=0$ and $x'\geq x$ when $x>0$.\footnote{$Inc(X)$ effects always increase variable $X$ when $X=0$. Otherwise, trajectory constraints would be needed also for increments.}
All the propositions in $F$ are assumed to be \emph{observable},
while for numerical variables $X$, only the booleans $X=0$ and $X>0$ are observable.

A QNP $R_V=\langle F,Init,Act,G,V,Init_V,INC,$ $DEC \rangle$ represents a PONDP $P_V$
in \emph{syntactic form}. If $s_V$ is a valuation over the variables in $V$, $s_V[X]$
is the value of $X$ in $s_V$, and $b[s_V]$ its boolean projection; i.e.\
$b[s_V][X]=0$ if $s_V[X]=0$ and $b[s_V][X]=1$ otherwise.
The PONDP $P_V$ corresponds to the tuple $\tup{S,I,\Omega,Act,T,A,obs,F}$ where
\begin{enumerate}[1.]\denselist
\item $S$ is the set of valuations $\tup{s_F,s_V}$ over the variables in $F$ and $V$ respectively,
\item $I$ is the  set of pairs $\tup{s_F,s_V}$ where $s_F$ is \emph{uniquely}
  determined by $Init$ and $s_V[X]\in Init_V(X)$ for each $X \in V$,
\item $\Omega=\{\tup{s_F,b[s_V]} : \tup{s_F,s_V}\in S\}$,
\item $Act$ is the finite set of \emph{actions} given,
\item $T \subseteq S$ is the set of states that satisfy the atoms in $G$,
\item $A(s)$ is the set of actions whose precondition are true in $s$, 
\item $obs(s)=\tup{s_F,b[s_V]}$ for $s=\tup{s_F,s_V}$,
\item $F(a,s)$ for $s=\tup{s_F,s_V}$ is the set of states $s'=\tup{s'_F,s'_V}$ where $s'_F$
  is the STRIPS successor of $a$ in $s$, and $s'_V[X]$ is in $INC(X,a,s)$, $DEC(X,a,s)$, or
  $\{s_V[X]\}$ according to whether $a$ features an $Inc(X)$ effect, a $Dec(X)$ effect, or none.
\end{enumerate}

An example  is $R_V=\langle F,Init,Act,G,V,Init_V,INC,$ $DEC\rangle$ where $F$ is empty, $V = \{X,Y\}$, $Init_V(X)=\{20\}$
and $Init_V(Y)=\{30\}$, the goal is $\{X=0,Y=0\}$, and $Act$ contains two actions: an action $a$ with effects $Dec(X)$ and $Inc(Y)$,
and an  action $b$ with effect $Dec(Y)$. The functions $INC$ always increases the value by 1 and $DEC$
decreases the value by 1 except when it is less than 1 when the value is decreased to $0$.
A policy $\mu$ that solves $R_V$ is  ``do $a$ if $X > 0$ and $Y=0$, and do $b$ if $Y > 0$''.
Simple variations of  $R_V$ can be obtained by changing the initial situation, e.g., setting it to $Init_V(X)=[10,20]$ and $Init_V(Y)=[15,30]$,
or the $INC$ and $DEC$ functions. For example, a variable $X$ can be set to increase or decrease any number between $0$ and $1$
as long as values do not become negative. The policy $\mu$ solves  all these variations except for those where  infinite sequences of decrements
fail to drive a variable to zero. We rule out such  QNPs by means of the following \emph{trajectory constraint}:


\begin{definition}[QNP Constraint]
\label{def:qnp:constraint}
The trajectory constraint $C_X$ for a numerical variable $X$ in a QNP excludes the trajectories
that after a given time point contain a) an infinite number of actions with $Dec(X)$ effects,
b) a finite number of actions with $Inc(X)$ effects,  c) no state where $X=0$.
\end{definition}

This is analogous to the LTL constraint in Section~\ref{sec:LTL}.
The set of   constraints $C_X$ for all variables $X \in V$ is denoted as $C_V$.
We  will be interested in solving QNPs $R_V$ \emph{given} that the constraints $C_V$ hold.
Moreover, as  we have seen, there are policies that solve entire
families of \emph{similar} QNPs:

\begin{definition}[Similar QNPs]
\label{def:qnp:similar}
Two QNPs are similar if they only differ in the $INC$, $DEC$ or $Init_V$ functions,
and for each variable $X$, $X=0$ is initially possible in one if it is initially possible in the other,
and the same for $X > 0$.
\end{definition}

We want to obtain policies that solve whole classes of similar QNPs by solving
a \emph{single abstract problem}. For a QNP $R_V$, we define its syntactic projection
as $R_V^o$:\footnote{For simplicity we  use the syntax of STRIPS with negation
  and conditional effects \cite{gazen:adl}.}

\begin{definition}[Syntactic Projection of QNPs]
\label{def:qnp:syntactic-projection}
If $R_V=\tup{F,Init,Act,G,V,Init_V,INC,DEC}$ is a QNP, its
\emph{syntactic projection} is the non-deterministic (boolean) problem
$R_V^o=\tup{F',Init',Act',G}$, where
\begin{enumerate}[1.]\denselist
\item $F'$ is $F$ with new atoms $X=0$ and $X>0$ added for each variable $X$;
  i.e.\ $F'=F\cup\{X=0,X>0: X\in V\}$,
\item  $Init'$ is $Init$ and  $X=0$ (resp.\ $X>0$) true iff $X>0$
  (resp.\ $X=0$) is not initially possible in $Init_V$.
\item $Act'$ is $Act$ but where in each action and for each variable $X$,
  the effect $Inc(X)$ is replaced by the atom $X>0$, and the effect $Dec(X)$
  is replaced by the conditional effect ``if $X>0$ then $X>0\,|\,X=0$''.
\end{enumerate}
\end{definition}
Recalling that $X>0$ is an abbreviation for $X\neq 0$, 
the atoms $X=0$ and $X>0$ are mutually exclusive.
We refer to the action with non-deterministic effects in $R_V^o$ as the $Dec(X)$
actions as such actions have $Dec(X)$ effects in $R_V$. This convention is assumed
when applying $C_V$ constraints to $R^o_V$.  The syntactic projection $R_V^o$ denotes a FONDP that features
multiple initial states when for some variable $X$, both $X=0$ and $X>0$
are possible in $Init_V$:

\begin{theorem}
\label{thm:qnp:basis}
The FONDP denoted by $R_V^o$ is the observation projection of the
class $\mathcal{R}_V$ made of all the PONDPs $R'_V$ that are similar
to $R_V$.
\end{theorem}
\if\withproofs1
\nosuchproof
\fi

The generalization captured by Theorem~\ref{thm:gen+constraints}
implies that:

\begin{theorem}[QNP Generalization]
\label{thm:qnp:generalization}
Let $\mu$ a policy that solves $R_V^o/C_V$; i.e.\ the FONDP denoted by $R_V^o$
given $C_V$. Then, $\mu$ solves $R'_V/C_V$ for all QNPs $R'_V$ similar to $R_V$.
\end{theorem}
\if\withproofs1
\begin{proof}
Direct application of Theorem~\ref{thm:gen+constraints} using the fact that the
class $\mathcal{R}_V$ is \alert{uniform} by Theorem~\ref{thm:qnp:basis}.
\end{proof}
\fi

\medskip

In addition, the constraints $C_V$ are strong enough in QNPs for making
the abstraction $R_V^o$ complete for the class of problems $\mathcal{R}_V$
similar to $R_V$: 

\begin{theorem}[QNP Completeness]
\label{thm:qnp:completeness}
Let $\mu$ be a policy that solves the class of problems $\mathcal{R}_V$
made up of all the  QNPs $R'_V$  that are similar to $R_V$ given $C_V$.
Then $\mu$ must solve the projection $R_V^o$ given $C_V$.
\end{theorem}

\noindent This is because the class $\mathcal{R}_V$ contains a problem $R^*_V$
where each variable $X$ has two possible values $X=0$ and $X=1$ such that the
semantics of the $Dec(X)$ and $Inc(X)$ effects 
make $R^*_V$ equivalent to the projection $R^o_V$, with the valuations
over the two-valued $X$ variables in correspondence with the boolean values $X=0$
and $X > 0$ in $R_V^o$.

\Omit{ 
The results above are very general and yet they can be extended in several ways.
For example, action preconditions and goals that feature atoms $X=0$ and $X > 0$ only,
can be extended to involve other value landmarks $L_X$ in the form of atoms $X=L_X$, $X < L_X$, etc. 
For this, the set of QNP constraints $C_X$ need to be extended and the semantics given
by the $INC$ and $DEC$ functions adjusted, so that no landmarks can be ``skipped over''.
Some of these extensions are considered in \cite{srivastava:aaai2015} that address
similar classes of problems but using a different formal framework. 
}

\if\withproofs1
\begin{proof}
Direct from Theorems~\ref{thm:basis:completeness} and \ref{thm:qnp:basis}.
\end{proof}
\fi

\subsection{QNP Solving as FOND Planning}

The syntactic projection $R_V^o$ of a QNP $R_V$ represents a FONDP with non-deterministic (boolean)
effects $X>0\,|\,X=0$ for the actions in $R_V$ with $Dec(X)$ effects.
It may appear from Theorem~\ref{thm:qnp:generalization} that one could
use off-the-shelf FOND planners for solving  $R_V^o$ and hence for solving
all QNPs similar  to $R_V$. There is, however,  an  obstacle: the effects $X>0\,|\,X=0$
are not   \emph{fair}.
Indeed, even executing forever only actions with $Dec(X)$ effects does not guarantee that eventually $X=0$ will be true.
In order to use fair (strong cyclic) FOND planners off-the-shelf,
we thus  need  to compile the FONDP $R^o_V$ \emph{given the constraints}  $C_V$ into a \emph{fair  FONDP with no constraints}.

For this, it is convenient to make two assumptions and to extend the  problem $R_V$ with extra booleans and actions that do not affect the problem
but provide us with handles in the projection. The assumptions are that actions with $Dec(X)$ effects have the (observable) precondition $X > 0$, and more critically,
that actions feature decrement effects for at most one variable. The new  atoms are  $q_X$, one for each variable $X \in V$, initially all false.
The new   actions for each variable $X$ in $V$ are  $set(X)$ and $unset(X)$, the first with no precondition and
effect $q_X$, the second with precondition $X=0$ and effect $\neg q_X$. Finally, preconditions $q_X$
are added to actions with effect $Dec(X)$ and  precondition $\neg q_X$  to all actions with effect $Inc(X)$.
Basically, $q_X$ is set in order to decrease the variable $X$ to zero. When $q_X$ is set,  $X$ cannot be
increased and $q_X$ can be unset only when $X=0$. We say that $R_V$ is \emph{closed} when $R_V$ is  extended in this way
and complies with the assumptions above (and likewise for its projection $R_V^o$).

\begin{theorem}[Generalization with FOND Planner]
  $\mu$ is a fair solution to the FONDP  $R^o_V$ for a closed QNP $R_V$ 
  iff $\mu$ solves all QNPs that are similar to $R_V$ given the constraints $C_V$.
\end{theorem}
\if\withproofs1
\nosuchproof
\fi

\noindent \emph{Sketch:} We need to show that  $\mu$ is a fair solution to $R^o_V$ iff $\mu$ solves $R^o_V/C_V$. 
The rest follows from Theorems~\ref{thm:qnp:generalization} and \ref{thm:qnp:completeness}. 
($\Rightarrow$). If $\mu$ does not solve the FONDP $R^o_V$ given $C_V$, 
there must a $\mu$-trajectory  $\tau$ that is not goal reaching but that satisfies $C_V$  and  is not fair in $R^o_V$.
Thus, there must be a subtrajectory $\tup{s_i,a_i,\ldots, s_{i+m}}$ that forms a loop with $s_{i+m}=s_i$, where
no $s_k$ is a goal state, and 1)~some $a_k$ has a $Dec(X)$ effect in $R_V$, and 2) $X > 0$ is true in all $s_k$, $i \leq k \leq i+m$.
1) must be true as $\tau$  is not fair in $R^o_V$ and  only actions with $Dec(X)$ actions in $R_V$ are not deterministic in $R_V^o$,
and 2) must be true as, from the assumptions in $R_V$, $X=0$ needs to be achieved by an action that decrements $X$, in contradiction with
the assumption that $\tau$ is not fair in $R_V^o$.  Finally, since $\tau$ satisfies $C_V$,  then it  must contain infinite  actions  with $Inc(X)$ effects, 
but  then the  loop must feature  $unset(X)$ actions with precondition $X=0$ in contradiction with 2. 
($\Leftarrow$) If $\mu$ solves $R^o_V/C_V$ but $\mu$ is not a fair solution to $R^o_V$, then there must be an infinite $\mu$-trajectory $\tau$
that is not goal reaching and does not satisfy $C_V$,  but which is fair in $R^o_V$. This means that there is a loop in $\tau$
with some $Dec(X)$ action, no $Inc(X)$ action, and where $X=0$ is false. But $\tau$ can't then be fair in $R^o_V$.\qed

\section{Discussion}

We have studied ways in which  a (possibly infinite) set of problems
with partial observations (PONDPs) that satisfy a set of trajectory constraints
can all be solved by solving  a \emph{single}  fully observable problem (FONDP)
given by the common observation projection, augmented with the trajectory constraints. 
The trajectory constraints play a crucial role in adding enough expressive power
to the  observation projection. The single  abstract problem  can be solved
with automata theoretic techniques typical of LTL synthesis when the trajectory constraints can be expressed in
LTL, and in some cases, by more efficient FOND planners. 

\Omit{
Generally speaking, our work aims at clarifying the core of generalized
planning: given a possibly infinite set of problems sharing actions
and observations, use their commonalities to generate a
single abstract problem to solve, and then apply the solution to all
concrete problems in the set.
}

The class of qualitative numerical problems  are related to those considered  by \citeauthor{srivastava:aaai2011}\ [\citeyear{srivastava:aaai2011,srivastava:aaai2015}]
although the theoretical foundations are different. We obtain the FONDPs $R_V^o$ from an explicit observation projection,
and rather than using FOND planners to provide \emph{fair} solutions to $R_V^o$ that are then checked for termination, 
we look for strong solutions to $R_V^o$ \emph{given}  a set of explicit trajectory constraints $C_V$, and show that under
some conditions, they can be obtained from fair solutions to a suitable transformed problem. 
It is also  interesting to compare our work to the approach adopted to deal with the one-dimensional planning problems
\cite{hu10correctness,hu:generalized}.
There we have infinitely many concrete problems that are all identical except for
an unobservable parameter ranging over naturals that can only decrease.
A technique for solving such generalized planning problems is based on
defining one single abstract planning problem to solve that is ``large
enough'' \cite{hu:generalized}. Here, instead, we would address such
problems by considering a much smaller abstraction, the observation
projection, but with trajectory constraints that  capture  that the hidden
parameter can only decrease.  

Our work is also relevant for planning under incomplete
information.  The prototypical example is the tree-chopping problem of felling a tree that requires an 
unknown/unobservable number of chops. This was studied by \citeay{SardinaGLL06},
where, in our terminology, the authors analyse exactly the issue of
losing the ``global property'' when passing to the ``observation 
projection''.
Finally, our approach can be seen as a concretization of 
insights by \citeay{DeGMRS16} where trace constraints are shown to be 
necessary for the belief-state construction to work on
infinite domains.

\section*{Acknowledgments}
H.\ Geffner was supported by grant TIN2015-67959-P, MINECO, Spain;
S.\ Rubin by a  Marie Curie fellowship of INdAM;  
and G.\ De Giacomo  by the Sapienza project 
``Immersive Cognitive Environments''.

\bibliographystyle{named}
\bibliography{control}

\end{document}